\documentclass{article}

\usepackage[utf8]{inputenc}
\usepackage[T1]{fontenc}
\usepackage{biblatex}
\usepackage{hyperref}
\usepackage{url}
\usepackage{booktabs}
\usepackage{amsmath}           
\usepackage{amsfonts}
\usepackage{amssymb}
\usepackage{amsthm}
\usepackage{microtype}
\usepackage{graphicx}
\usepackage{tikz}
\usepackage{geometry}
\theoremstyle{plain}
\newtheorem{proposition}{Proposition}
\newtheorem{corollary}[proposition]{Corollary}
\theoremstyle{remark}
\newtheorem{remark}[proposition]{Remark}

\newcommand{\egirth}[1]{g_{#1}}
\newcommand{\mult}[1]{\lambda_{#1}}
\newcommand{\egseq}{\sigma}
\newcommand{\sigm}{\mathrm{sigm}}   
\newcommand{\bh}{\mathbf{h}}
\newcommand{\be}{\mathbf{e}}
\newcommand{\bbm}{\mathbf{m}}       
\newcommand{\bg}{\mathbf{g}}
\newcommand{\EGAGNN}{\textsc{EGAGNN}}

\date{}
\author{
Lilian Marey\thanks{LTCI, Télécom Paris, Institut Polytechnique de Paris, Palaiseau, France (\texttt{lilian.marey@telecom-paris.fr}).}
\and
Charlotte Laclau\thanks{LTCI, Télécom Paris, Institut Polytechnique de Paris, Palaiseau, France (\texttt{charlotte.laclau@telecom-paris.fr}).}
}
\title{Edge-Girth as a Structural Edge Feature for Graph Neural Networks}

\begin{document}
\maketitle

\begin{abstract}
Graph neural networks (GNN) based on message passing are provably no more powerful
than the one-dimensional Weisfeiler--Leman colour-refinement test (1-WL):
starting from a uniform colouring, each node is repeatedly recoloured as a
function of its own colour and the multiset of its neighbours' colours, and
two graphs the process cannot tell apart receive identical representations,
however deep or wide the network. A common remedy augments node or edge
features with precomputed structural descriptors, most often counts of a
fixed small subgraph such as triangles or longer cycles, but such counts
require committing in advance to the size of the substructure being counted,
a choice usually made blind to the data. We study a descriptor that avoids
this choice. The \emph{edge-girth} of an edge is the length of a shortest
cycle through it, and its \emph{multiplicity} is the number of such shortest
cycles; together they form a per-edge invariant that reports cycles of
arbitrary length and is computable exactly by a single breadth-first search
per edge. Like other structural encodings, it is computed once from the graph
alone, before and independently of any supervision, and is therefore not
specific to a downstream objective. We evaluate it on two tasks of different
kinds. 
Injected into a gated message-passing architecture, \EGAGNN{}, it
reaches a test MAE a factor three below the closest gated comparator on the
\textsc{Zinc}-12k regression benchmark at $104$k parameters; compared against
bounded cycle-counting descriptors under the same architecture, it matches
only a dictionary that explicitly counts cycles up to length eight, which
requires twice as many channels, while a dictionary capped at length four
performs no better than no structural information at all.
On graph discrimination we prove a matching limitation: on
graphs where every edge sees the same number of shortest cycles of the same
length, the descriptor becomes constant and any model built on it collapses
back to the 1-WL bound. This prediction holds without exception across all 400 pairs of the BREC graph-isomorphism benchmark: not one of the 90 such pairs is distinguished.
\end{abstract}

Code and datasets are available at 
\url{https://anonymous.4open.science/r/GDDL_EGAGNN-5154}

\section{Introduction}
\label{sec:intro}

Message-passing neural networks (MPNNs) are limited in expressive power by the
one-dimensional Weisfeiler-Leman test (1-WL): two graphs that $1$-WL fails to
separate receive identical representations, no matter how many layers or
parameters are used \cite{xu2018powerful,morris2019weisfeiler}. 
A standard remedy is to enrich node or edge inputs with precomputed structural
descriptors, most commonly counts of small subgraphs \cite{bouritsas2022improving}
or random-walk and spectral encodings \cite{dwivedi2021graph}. Higher-order
approaches such as cellular complexes \cite{bodnar2021weisfeiler} make cycle structure explicit at the message-passing level, while recent work on the cycle counting expressiveness of MPNNs \cite{huang2022boosting} has quantified the limits of cycle-based augmentation.
Subgraph counting carries an assumption that is
rarely made explicit: the informative substructure must be \emph{bounded in
size and fixed in advance}. The limitation is one of coverage rather than of
budget. A dictionary of motifs of size at most $k$ cannot report a cycle of
length $k+1$, however much computation is spent on it; enlarging $k$ trades
one ceiling for a higher one but never removes it. The choice is consequential
and is made blind: on a molecular dataset such as \textsc{Zinc}, a per-edge cycle dictionary capped at
$k=4$ leaves the model no better than one given no structural descriptor at
all, while the same dictionary at $k=8$ recovers almost all of the available
gain.

In this article, we study a descriptor that does not require this choice. The \emph{edge-girth}
$\egirth{e}$ of an edge $e$ is the length of a shortest cycle containing $e$,
with $\egirth{e} = \infty$ when no such cycle exists; alongside it, we carry
$\mult{e}$, the number of cycles of that length through $e$. The pair is
attached to each edge, is invariant under isomorphism, and reports a cycle of
unbounded length: an edge lying only on a $20$-cycle is described as
faithfully as one lying on a triangle, with no dictionary to enlarge. It is
also exactly computable, by a single breadth-first search per edge: the
quantity equals one plus the length of a shortest path between the endpoints
of $e$ that avoids $e$, the object of the replacement path problem
\cite{bernstein2010nearly}. It carries meaning in applications, governing
transition rates in active flow networks \cite{woodhouse2016stochastic}, and
short cycles through an edge -- exactly the small edge-girth values -- are a
dominant cause of decoding failure in the Tanner graphs of LDPC codes
\cite{xu2025ldpc}. The collection of these values over all edges, the
\emph{edge-girth sequence}, has recently been characterised up to
realizability \cite{marey2026realizability}.

That characterisation also indicates how the descriptor should \emph{not} be
used: the sequence is additive under vertex identification and does not
determine the number of vertices of the graph, so it is a poor graph
signature. We therefore propose to treat edge-girth locally, injecting the per-edge pair into message passing so that it modulates the flow of node information rather
than summarising the graph on its own. The resulting model, \EGAGNN{} (Edge-Gated Aligned GNN), gates
each message by a learned function of the edge it travels along, and lets edge
states absorb the context of their endpoints so that edge-girth information
diffuses beyond immediate neighbourhoods.

The descriptor is effective on molecular property prediction and provably
powerless on a precisely identifiable family of graphs. On \textsc{Zinc}-12k
at a matched parameter budget, \EGAGNN{} reaches a test MAE of
$0.0932 \pm 0.0035$, a factor three below the closest gated comparator.
Holding the architecture and the chemical inputs fixed and varying only the
structural descriptor, we compare $(\egirth{e}, \mult{e})$ (together with a
bridge indicator, three channels per edge in total) against bounded
cycle-count dictionaries of increasing length. The three-channel descriptor is
matched only by a six-channel dictionary reaching length eight, and
dictionaries capped at four are worthless on this data: the advantage lies
less in what edge-girth sees than in not having to guess where to stop
looking. In the opposite direction, on \emph{edge-girth-regular} graphs (those in which every edge lies on the same number
of shortest cycles, all of the same length, \cite{jajcay2018edge}) the descriptor is constant by
definition, and any model built on it computes what it would compute from a
constant edge input: its expressive power falls back to the $1$-WL bound. The
prediction holds without a single exception on the BREC benchmark
\cite{wang2023empirical}, a purpose-built collection of graph pairs designed
to be hard for the Weisfeiler--Leman hierarchy. It does not, however, account
for every hard case in that benchmark, as we discuss in
Section~\ref{sec:discussion}.

\paragraph{Contributions.}
\begin{itemize}
  \item We propose the per-edge pair $(\egirth{e}, \mult{e})$ as a structural
        descriptor requiring no motif-size budget and computable exactly by
        one breadth-first search per edge, and \EGAGNN{}, a gated
        message-passing architecture that consumes it
        (Sections~\ref{sec:edgegirth}--\ref{sec:egagnn}).
  \item On \textsc{Zinc}-12k with all architectures matched to
        ${\approx}100$k parameters, \EGAGNN{} reaches $0.0932 \pm 0.0035$ test
        MAE, and a depth-matched comparator rules out nonlinear depth as the
        explanation. A descriptor-versus-descriptor study, holding
        architecture and chemical inputs fixed, quantifies what a bounded
        cycle dictionary must reach to match an unbounded one
        (Sections~\ref{subsec:zinc}--\ref{subsec:descriptors}).
  \item We prove that edge-girth-based models degenerate onto their backbone
        on edge-girth-regular graphs, and verify the prediction pair by pair:
        of the $400$ BREC pairs, $90$ are edge-girth-regular and not one is
        resolved. We further show that the multiplicity $\mult{e}$ is
        necessary for the statement -- an edge-girth-only version of it would
        be false (Section~\ref{sec:expressivity}).
\end{itemize}

\section{Edge-Girth as a Structural Edge Invariant}
\label{sec:edgegirth}
 
\subsection{Edge-Girth and Multiplicity, by Example}
\label{subsec:defs}
 
Let $G = (V, E)$ be a simple connected graph. The \emph{edge-girth} of an edge
$e \in E$, denoted $\egirth{e}$, is the minimum number of edges of a simple
cycle of $G$ containing $e$, and $\egirth{e} = \infty$ when $e$ lies on no
cycle, i.e.\ when $e$ is a bridge. Edge-girth values range over
$\{3, 4, 5, \dots\} \cup \{\infty\}$, and $\min_{e \in E} \egirth{e}$ recovers
the usual girth of $G$. We write $\mult{e}$ for the \emph{multiplicity} of
$e$, the number of distinct cycles of length $\egirth{e}$ through $e$. The
pair $(\egirth{e}, \mult{e})$ is the elementary object this work builds on: it
is defined per edge, invariant under graph isomorphism, and, unlike a
subgraph count, carries no bound on the size of the structure it reports.
 
Figure~\ref{fig:example} works through both quantities on a small graph.
 
%

\begin{figure}[t]
  \centering
  \begin{tikzpicture}[
      scale=0.95,
      vtx/.style   = {circle, draw=black, fill=black, inner sep=1.5pt},
      cyc/.style   = {fill=black!7, draw=none},
      ed/.style    = {black!70, line width=0.7pt},
      hl/.style    = {line width=1.6pt},
      mark/.style  = {circle, draw=black, fill=white, inner sep=0.9pt,
                      font=\tiny\bfseries, minimum size=3.4mm}]

    \coordinate (a) at (0, 2);
    \coordinate (b) at ( 2, 2);
    \coordinate (c) at ( 0, 0);
    \coordinate (d) at ( 2, 0);
    \coordinate (e) at ( 4, 2);
    \coordinate (f) at ( 4, 0);
    \coordinate (g) at ( -2, 2);

    \draw[ed] (a) -- (b) -- (c)-- (a);
    \draw[ed] (b) -- (c) -- (d) -- (b);
    \draw[ed] (b) -- (d) -- (f) -- (e) -- (b);

    \draw[hl, red!65!black]   (a) -- (b);
    \draw[hl, blue!60!black]  (b) -- (c);
    \draw[hl, green!45!black] (b) -- (d);
    \draw[hl, violet!45!black] (f) -- (e);
    \draw[hl, black!55, dashed, dash pattern=on 4pt off 2.5pt] (a) -- (g);

    \foreach \p in {a,b,c,d,e,f, g} { \node[vtx] at (\p) {}; }

    \node[text=black!45!black] at ( -1, 2.3) {a};
    \node[text=red!65!black]   at ( 1, 2.3) {b};
    \node[text=blue!60!black]  at (1, .7) {c};
    \node[text=green!45!black] at ( 2.3, 1) {d};
    \node[text=violet!45!black] at ( 4.3, 1) {e};
  \end{tikzpicture}

  \smallskip
  {\small
   \setlength{\tabcolsep}{5pt}
   \begin{tabular}{@{}lllll@{}}
  \textbf{a}: $\egirth{e} = \infty$ &
  \textbf{b}: $\egirth{e} = 3$, $\mult{e} = 1$ &
  \textbf{c}: $\egirth{e} = 3$, $\mult{e} = 2$ &
  \textbf{d}: $\egirth{e} = 3$, $\mult{e} = 1$ &
  \textbf{e}: $\egirth{e} = 4$, $\mult{e} = 1$
\end{tabular}}

  \caption{The descriptor on a small graph, illustrated on five edges.
    Edge~\textbf{a} is a bridge: it lies on no cycle, so its edge-girth is
    infinite. Edge~\textbf{b} lies on a single triangle, and edge~\textbf{c}
    on two, which is what its multiplicity of two records, the
    multiplicity counts how many shortest cycles an edge belongs to, not how
    many cycles in total. Edge~\textbf{d} is the instructive case: it lies on
    a triangle and also on the square formed with edges~\textbf{e} and its
    neighbours, and only the shorter of the two is reported, so its
    edge-girth is three rather than four. Edge~\textbf{e} belongs to no
    triangle at all, and its shortest cycle is that square, giving
    edge-girth four.}
  \label{fig:example}
\end{figure}
 
Collecting the edge-girth values over all edges yields the \emph{edge-girth sequence} $\egseq(G)$.
Which sequences arise as $\egseq(G)$ for some simple connected $G$ has recently been settled: a sequence is realizable if and only if it satisfies a recursive criterion on the number of edges attaining its largest value, controlled by the maximum diameter attainable by graphs realizing the sequence with that largest value removed \cite{marey2026realizability}. We use that characterisation only as
background; the two properties we need are stated next. A related but distinct
notion is that of \emph{edge-girth-regular} graphs \cite{jajcay2018edge}:
regular graphs where every edge lies on exactly $\lambda$ shortest cycles, all of length $g$, so that $(\egirth{e}, \mult{e})$ is constant across edges. The formal definition is given in Section~\ref{sec:expressivity}, which shows what this family delimits.
 
\subsection{The Edge-Girth Sequence Is a Weak Global Invariant}
\label{subsec:weakness}
 
Two properties show that $\egseq$ is poorly suited to acting as a graph
signature on its own.
 
First, $\egseq$ is \emph{additive under vertex identification}: gluing two
graphs $G_1$ and $G_2$ at a single vertex (denoted as $G_1 \oplus G_2$)
creates no new cycle, so the edge-girth values of $G_1 \oplus G_2$ are
exactly those of $G_1$ together with those of $G_2$ \cite{marey2026realizability}. The sequence is blind to
how components are attached to one another, and any two assemblies of the same
building blocks are indistinguishable by $\egseq$.
 
Second, $\egseq$ does not determine the order of the graph: the sequence 
consisting of nine edge-girth values all equal to $3$ is realized both by a 
graph on five vertices and by a graph on seven vertices 
\cite{marey2026realizability}. A quantity that fails to recover $|V(G)|$ 
cannot be expected to separate non-isomorphic graphs in general.
 
We therefore attach $(\egirth{e}, \mult{e})$ to each edge and inject it into
message passing rather than using $\egseq$ as a graph-level descriptor. This
buys nothing on the isomorphism benchmark of Section~\ref{subsec:brec}, where
propagating the descriptor through a network matches a direct hash of the
per-edge multiset to within one pair out of four hundred; it pays off on
graph-level regression (Section~\ref{subsec:zinc}), and
Section~\ref{sec:expressivity} explains why no architecture built on this
descriptor could have done better on the former.
 
\subsection{Computing $(\egirth{e}, \mult{e})$ by Breadth-First Search}
\label{subsec:complexity}
 
Both quantities come from a single traversal per edge. For $e = \{u,v\}$, a
breadth-first search from $u$ in $G \setminus \{e\}$ returns $d(u,v)$, the shortest path distance between $u$ and $v$ in $G \setminus \{e\}$, and by the standard shortest-path counting recursion, the number of
shortest $u$--$v$ paths.
Since such a path is simple and avoids $e$, closing it with $e$ gives $\egirth{e} = d(u,v) + 1$ and $\mult{e}$ exactly, with
$\egirth{e} = \infty$ when $v$ is unreachable. No approximation and no
truncation is involved.
 
The traversal is paid once per edge, so the descriptor costs
$O(|E|(|V| + |E|))$ per graph. This is a real cost, and we make no claim that
it is smaller than that of bounded motif counting: in our measurements the
ranking between the two reverses with the graph family. On molecular graphs,
which are small and comparatively dense, the per-edge traversal is the cheaper of the two; on sparse, low-degree graphs whose edges lie on long cycles, a
bounded cycle enumeration has few cycles to find and wins comfortably. 
What the traversal buys is exactness and the absence of a size ceiling, 
not speed. Unlike higher-order GNNs that elevate cycles to the message-passing
level \cite{bodnar2021weisfeiler}, edge-girth keeps the backbone standard 
while enriching edge features, trading expressivity for computational simplicity.
This complements recent analyses of cycle-counting expressiveness 
\cite{huang2022boosting}.
Computing $\egirth{e}$ is an instance of the replacement path problem, for
which sharper bounds are available \cite{bernstein2010nearly}; the
linear-time procedure of \cite{goedgebeur2025exhaustive} applies when the
girth is globally constant, which is not the regime considered here.
 
The computation is a one-off preprocessing step, cached and amortised over
training in the same way as random-walk and Laplacian encodings
\cite{dwivedi2021graph}. At molecular scale it costs a fraction of a single
training run and is reused across every seed and every hyperparameter setting,
so it does not enter the comparison between methods.
 
\section{\EGAGNN}
\label{sec:egagnn}

\subsection{Edge Feature Initialization}
\label{subsec:initfeat}
 
Each edge $e = \{u,v\}$ is endowed with the pair introduced in
Section~\ref{sec:edgegirth}. Two adjustments make it usable as a network
input. Bridges, for which $\egirth{e} = \infty$, cannot be handed to a network
as a finite number, so we replace the pair by a fixed placeholder value
together with an explicit binary indicator: the network is told directly that
the edge lies on no cycle, rather than being given an arbitrarily large finite
substitute that it might mistake for a very long cycle. Concretely, both
quantities are normalised using training-split statistics only, and a bridge
is represented by the placeholder value
$(\tilde g_{uv}, \tilde \lambda_{uv}) = (0, 0)$ together with the indicator
set to $1$. The initial edge representation is
\begin{equation}
  \be_{uv}^{(0)} \;=\;
  \Big[\, \tilde{g}_{uv}, \;\; \tilde{\lambda}_{uv}, \;\;
          \mathbf{1}\{\egirth{uv} = \infty\} \,\Big]
  \;\;\big\|\;\; \mathbf{a}_{uv},
  \label{eq:edgeinit}
\end{equation}
where $\tilde{g}$ and $\tilde{\lambda}$ are the normalised edge-girth and
multiplicity, and $\mathbf{a}_{uv}$ denotes any edge attributes supplied by
the dataset, such as bond types on molecular graphs. On the isomorphism
benchmark $\mathbf{a}_{uv}$ is empty and the structural triple is the only
edge input; on \textsc{Zinc} it is not, which is why
Section~\ref{subsec:descriptors} holds the architecture and $\mathbf{a}_{uv}$
fixed and varies only the structural part.
 
Both entries $\tilde{g}$ and $\tilde{\lambda}$ come from the single breadth-first
search of Section~\ref{subsec:complexity}. We deliberately stop at the
shortest cycle. Going further means one of two things. Counting simple cycles
of \emph{unbounded} prescribed length through an edge is
$\#\mathsf{W}[1]$-hard~\cite{flum2004parameterized}, whereas counting shortest ones is a linear-time
recursion. 
Counting them up to a fixed length $k$ is tractable, it is
exactly what a bounded cycle dictionary does. Section~\ref{subsec:descriptors}
measures what that choice costs relative to the unbounded descriptor.

\subsection{Gated Message Passing}
\label{subsec:gating}

Structural edge information is injected multiplicatively rather than by
concatenation to node features, so that an edge modulates what passes through
it. At layer $\ell$, the message sent from $v$ to $u$ is
\begin{equation}
  \bbm_{uv}^{(\ell)} \;=\;
  \sigm\!\Big( \phi_e^{(\ell)}\big( \be_{uv}^{(\ell-1)} \big) \Big)
  \;\odot\;
  \phi_n^{(\ell)}\big( \bh_{v}^{(\ell-1)} \big),
  \label{eq:message}
\end{equation}
where $\phi_e^{(\ell)}$ and $\phi_n^{(\ell)}$ are MLPs acting on the edge and
node states respectively, $\sigm$ is the sigmoid
and $\odot$ the Hadamard product. The sigmoid output acts as a learned
per-channel gate: an edge on a short cycle and an edge on a long one open
different channels, and the model learns which edge-girth regimes are worth
propagating for the task at hand.

Node and edge states are then updated by
\begin{align}
  \bh_{u}^{(\ell)}  &= \bh_{u}^{(\ell-1)}
     + \mathrm{AGG}\Big( \big\{ \bbm_{uv}^{(\ell)} \;\big|\;
       v \in \mathcal{N}(u) \big\} \Big),
     \label{eq:nodeupdate} \\[2pt]
  \be_{uv}^{(\ell)} &= \psi^{(\ell)}\Big( \be_{uv}^{(0)} \,\big\|\,
     \be_{uv}^{(\ell-1)} \,\big\|\, \bh_{u}^{(\ell)} \,\big\|\,
     \bh_{v}^{(\ell)} \Big),
     \label{eq:edgeupdate}
\end{align}
with $\psi^{(\ell)}$ an MLP and $\|$ concatenation.
Equation~\eqref{eq:nodeupdate} is a residual aggregation of the gated
messages. Equation~\eqref{eq:edgeupdate} is what makes the scheme more than a
static feature augmentation. Because an edge state absorbs the representations
of its two endpoints at every layer, and those endpoints have themselves
aggregated their neighbourhoods, after $\ell$ layers the gate on $\{u,v\}$
depends on the edge-girth values of every edge within $\ell$ hops, not only on
$\egirth{uv}$. Re-injecting $\be_{uv}^{(0)}$ at each layer keeps the raw
descriptor available and prevents it from being washed out by this diffusion.

Two consequences matter for Section~\ref{sec:expressivity}. First, the
architecture is a message-passing scheme: absent any edge input it inherits
the $1$-WL bound. Second, the \emph{only} information it receives beyond the
adjacency structure is $\be_{uv}^{(0)}$. Whenever that input carries no
discriminative signal, the model can do no better than its backbone.

\subsection{Graph-Level Readout}
\label{subsec:readout}

A permutation-invariant pooling of the final node states yields the graph
representation $\bg = \mathrm{READOUT}\big(\{\bh_u^{(L)} \mid u \in V\}\big)$,
passed to a regression head on \textsc{Zinc}; on the isomorphism benchmark the
graph embeddings themselves are compared under the protocol of
Section~\ref{subsec:brec}.

\section{Expressivity: What Edge-Girth Cannot Do}
\label{sec:expressivity}

The descriptor of Section~\ref{sec:edgegirth} is informative only inasmuch as
it varies across the edges of a graph. There is a family on which it does not
vary at all, and on which the model retains no advantage over its backbone.

Recall from \cite{jajcay2018edge} that a graph is
$\mathrm{egr}(n, k, g, \lambda)$, or \emph{edge-girth-regular}, if it is
$k$-regular on $n$ vertices, has girth $g$, and every edge lies on exactly
$\lambda$ cycles of length $g$. In such a graph the descriptor is constant by
construction, and this is enough to erase it.

\begin{proposition}
\label{prop:degeneracy}
Let $G_1$ and $G_2$ be edge-girth-regular graphs with the same parameters
$(g, \lambda)$ and no dataset edge attributes. Then \EGAGNN{} assigns them
distinct graph representations only if the $1$-WL test distinguishes them.
\end{proposition}

\begin{proof}
Every edge of an edge-girth-regular graph satisfies $\egirth{e} = g$ and
$\mult{e} = \lambda$, and no edge is a bridge, so
Equation~\eqref{eq:edgeinit} gives $\be_{uv}^{(0)} = \mathbf{c}$ for every edge
of either graph, with the same constant $\mathbf{c}$ in both. We show by
induction on $\ell$ that $\bh_u^{(\ell)}$ is a function of the $1$-WL colour
$c_\ell(u)$ and that $\be_{uv}^{(\ell)}$ is a function of
$\{\!\{c_\ell(u), c_\ell(v)\}\!\}$. This holds at $\ell = 0$, node states being
initialised identically and $\be^{(0)}$ being constant. Assuming it at
$\ell - 1$, the message \eqref{eq:message} is a function of
$\big(\{\!\{c_{\ell-1}(u), c_{\ell-1}(v)\}\!\}, c_{\ell-1}(v)\big)$;
aggregating over $\mathcal{N}(u)$ in \eqref{eq:nodeupdate} yields a function of
$c_{\ell-1}(u)$ and of the multiset
$\{\!\{c_{\ell-1}(v)\}\!\}_{v \in \mathcal{N}(u)}$, which is exactly the
refinement defining $c_\ell(u)$; the edge update \eqref{eq:edgeupdate} then
depends only on $\mathbf{c}$ and on the colours of its endpoints. As
$\mathrm{READOUT}$ is a function of the multiset of final node states, two
graphs with identical $1$-WL colour multisets receive identical
representations.
\end{proof}

The gates do not stay constant across layers: from the first layer on, they
absorb node states through \eqref{eq:edgeupdate}. What the proof shows is simpler. The model computes exactly the function it would compute with \emph{any} constant edge input, so its expressive power is that of its
message-passing backbone, and the $1$-WL bound applies unchanged.

\begin{corollary}
\label{cor:families}
Strongly regular graphs with $\lambda > 0$, and among them those satisfying
the four-vertex condition, are edge-girth-regular: such a graph has girth $3$
with exactly $\lambda$ triangles through every edge. Distance-regular graphs
are likewise edge-girth-regular, the girth and the number of shortest cycles
through an edge being determined by the intersection array
\cite{brouwer2011distance,van2014distance}. Pairs drawn from either family
with matching parameters therefore fall under
Proposition~\ref{prop:degeneracy}.
\end{corollary}

\begin{remark}[the multiplicity is necessary]
\label{rem:lambda}
Proposition~\ref{prop:degeneracy} requires the full pair
$(\egirth{e}, \mult{e})$ to be constant, and this is not a technical
convenience. Requiring only $\egirth{e}$ to be constant yields a strictly
larger family (on BREC, $103$ pairs instead of $90$), over which the
proposition would be \emph{false}: of the $13$ additional pairs, $12$ are
resolved by \EGAGNN{}, precisely because $\mult{e}$ varies on them while
$\egirth{e}$ does not. Carrying the multiplicity keeps the blind spot minimal.
\end{remark}

\begin{remark}[generality]
\label{rem:generality}
Nothing in the proof is specific to \EGAGNN{}. The argument applies to any
architecture whose sole structural input is a per-edge function of
$(\egirth{e}, \mult{e})$, and in particular to a plain MPNN augmented with
these values as edge features. The limitation is a property of the descriptor,
not of the way it is consumed. 
Escaping it requires strictly more information, cycle counts at lengths beyond the girth, for instance, which
reintroduces the trade-off of Section~\ref{subsec:initfeat}.
\end{remark}

Proposition~\ref{prop:degeneracy} is a prediction that can be read off a
benchmark before any model is trained, and it is falsifiable pair by pair:
every pair whose two graphs are edge-girth-regular with matching parameters
must go unresolved. Section~\ref{subsec:brec} tests it on all $400$ pairs of
BREC, and Section~\ref{sec:discussion} discusses how much of the benchmark's
difficulty it accounts for.

\section{Experiments}
\label{sec:experiments}

\subsection{Setup}
\label{subsec:setup}
 
The descriptor is computed from the graph alone, once, before any training,
and does not depend on the downstream objective, it is a property of the
representation rather than of the task. We therefore evaluate it on two tasks
of different kinds, which probe it in complementary ways. \textsc{Zinc}-12k
\cite{dwivedi2023benchmarking} measures whether the descriptor carries signal
a supervised model can exploit, and lets us compare it against alternative
structural descriptors under a fixed architecture. BREC
\cite{wang2023empirical}, a benchmark of graph pairs constructed to be
indistinguishable by $1$-WL and difficult for higher levels of the
Weisfeiler--Leman hierarchy, measures separating power directly, with no
objective to fit, and is where the limitation of
Section~\ref{sec:expressivity} can be tested pair by pair. Results on CSL are
relegated to Appendix~\ref{app:csl}, where two of the reference
implementations proved unreliable.
 
On \textsc{Zinc} we use the official splits ($10{,}000/1{,}000/1{,}000$), train
for $200$ epochs and report test MAE averaged over four seeds, with no early
stopping and no test-set model selection. All architectures are matched to a
budget of $100{,}000$ parameters within $\pm 10\%$ by adjusting the hidden
width, so that comparisons are not confounded by capacity. On BREC we use the
official reliable-paired-comparison protocol with $32$ relabelings, on all
$400$ pairs; every method attains a $100\%$ reliability rate, i.e.\ no method
is credited with distinguishing a pair it cannot reliably separate under
permutation. All experiments run on CPU. Full hyperparameters are given in
Appendix~\ref{app:details}.
 
\paragraph{Calibration.}
Three independent checks place our pipeline against published numbers. The $2$-dimensional Folklore Weisfeiler--Leman colour refinement (Exact
$2$-FWL), provably equivalent in distinguishing power to $3$-WL, reaches $0.675$
overall on BREC, matching the published figure for $3$-WL on this benchmark;
our implementation of PPGN \cite{maron2019provably}, a message-passing
architecture proven to match $2$-FWL's power, resolves $50$ of the $140$ pairs
that \cite{wang2023empirical} report as a single aggregated
\textsc{Regular} bucket, i.e.\ $35.7\%$, exactly the published value once the
same aggregation is applied; and on \textsc{Zinc} our GatedGCN baseline
reaches $0.278$ against a published $0.375 \pm 0.003$ at a comparable budget
\cite{dwivedi2023benchmarking}. The baselines we compare against are at least
as strong as their reference implementations, not weaker.

CW Networks \cite{bodnar2021weisfeiler} and I$^2$-GNN
\cite{huang2022boosting} were left out of Table~\ref{tab:zinc} for
implementation cost rather than principle; GSN remains the closest
descriptor-augmented comparator we report.

\subsection{Graph Regression on \textsc{Zinc}}
\label{subsec:zinc}

\begin{table}[t]
  \caption{\textsc{Zinc}-12k test MAE (lower is better), mean $\pm$ std over
           four seeds, all architectures matched to $\approx\!100$k
           parameters. GatedGCN-MLP replaces each linear transform of GatedGCN
           by the same two-layer MLP \EGAGNN{} uses, matching nonlinear depth
           per layer. GSN here is a GIN backbone with per-node orbit counts up
           to length six and bond-type edge features; it differs from
           \EGAGNN{} in architecture as well as descriptor and is reported as
           an external baseline, not as a controlled comparison.
           Section~\ref{subsec:descriptors} provides the latter.}
  \label{tab:zinc}
  \centering
  \small
  \begin{tabular}{lccc}
    \toprule
    Method & Hidden dim. & Params & Test MAE \\
    \midrule
    \EGAGNN{}       & 54  & 104{,}113 & $\mathbf{0.0932 \pm 0.0035}$ \\
    GSN ($k=6$)     & 98  & 102{,}313 & $0.1725 \pm 0.0129$ \\
    GatedGCN-MLP    & 54  & 102{,}601 & $0.2726 \pm 0.0151$ \\
    GatedGCN        & 74  & 102{,}121 & $0.2778 \pm 0.0072$ \\
    GIN             & 98  & \phantom{0}99{,}961 & $0.3152 \pm 0.0108$ \\
    GCN             & 130 & 106{,}081 & $0.4372 \pm 0.0096$ \\
    \bottomrule
  \end{tabular}
\end{table}

Table~\ref{tab:zinc} reports \textsc{Zinc} results at matched capacity.
GCN, GIN, GatedGCN and GatedGCN-MLP do not consume bond types in our
implementation, whereas \EGAGNN{} and GSN do (Appendix~\ref{app:asymmetries}). The comparison to those four methods therefore reflects architecture,
descriptor, and edge-attribute access jointly, and
Section~\ref{subsec:descriptors} is the controlled comparison. \EGAGNN{}
reaches $0.0932 \pm 0.0035$, a factor $3.0$ below GatedGCN, the closest
architectural comparator, and a factor $1.9$ below GSN with a cycle
dictionary sized for molecular rings. For reference, the published GSN
figures are $0.115 \pm 0.012$ at a comparable budget and
$0.101 \pm 0.010$ with roughly five times more parameters
\cite{bouritsas2022improving}.

\EGAGNN{}'s per-layer transforms are two-layer MLPs whereas GatedGCN's are
single linear maps, so matched parameter counts do not by themselves match
nonlinear depth. We therefore built GatedGCN-MLP, identical to GatedGCN except
that each of its four linear transforms is replaced by the same two-layer MLP
used in \EGAGNN{}, holding the gating mechanism, aggregation and skip
connection fixed. It moves the MAE by $0.005$, within the combined standard
deviation: nonlinear depth is not the explanation.

\subsection{Which Structural Descriptor?}
\label{subsec:descriptors}

Table~\ref{tab:zinc} varies architecture and descriptor together, so it cannot
attribute the improvement to either. We therefore hold the architecture and
the bond attributes fixed -- every variant below is \EGAGNN{} with the same
capacity and the same chemistry -- and vary only the structural entries of
$\be_{uv}^{(0)}$. Alongside $(\egirth{e}, \mult{e})$ we consider per-edge
counts of simple cycles of each length up to a ceiling $k$, the natural
bounded-dictionary counterpart to an unbounded descriptor. These counts are
computed per edge and injected identically; they are not GSN, whose per-node
orbit counts and backbone differ, though the bounded enumeration mechanism is
the same idea \cite{bouritsas2022improving}.

\begin{table}[t]
  \caption{\textsc{Zinc}, architecture and bond attributes held fixed, only
           the structural edge descriptor varied. ``Dim'' is the number of
           structural channels supplied per edge. Mean $\pm$ std over four
           seeds at a matched parameter budget.}
  \label{tab:descriptors}
  \centering
  \small
  \begin{tabular}{lcc}
    \toprule
    Structural descriptor & Dim & Test MAE \\
    \midrule
    none (constant)                   & 0 & $0.2044 \pm 0.0030$ \\
    triangle count                    & 1 & $0.2154 \pm 0.0150$ \\
    cycle counts, lengths $3$--$4$    & 2 & $0.2021 \pm 0.0117$ \\
    cycle counts, lengths $3$--$6$    & 4 & $0.1150 \pm 0.0124$ \\
    cycle counts, lengths $3$--$8$    & 6 & $0.1005 \pm 0.0090$ \\
    $(\egirth{e}, \mult{e})$ + bridge & 3 & $\mathbf{0.0932 \pm 0.0035}$ \\
    \bottomrule
  \end{tabular}
\end{table}

Table~\ref{tab:descriptors} makes the cost of the ceiling explicit. A
dictionary capped at length four is worth nothing on this data: at $0.2021$ it
is indistinguishable from supplying no structural descriptor at all, and
counting triangles alone is no better. The picture changes once the ceiling
clears the ring sizes that occur in drug-like molecules, with length six
reaching $0.1150$ and length eight $0.1005$. Edge-girth reaches $0.0932$ with
three channels rather than six, and without the ceiling ever being chosen.

The gap between the best bounded dictionary and edge-girth is small, and we
claim no more from it than it supports: on molecular graphs, a cycle
dictionary sized to $k=8$ comes close. What the comparison shows is the shape
of the trade-off. The useful ceiling is not knowable in advance, as $k=4$ is
worthless here and $k=6$ leaves a fifth of the gain on the table, and it is
dataset-specific, whereas the unbounded descriptor requires no such choice and
reports the informative cycle whatever its length.


\subsection{Isomorphism Discrimination on BREC}
\label{subsec:brec}

\begin{table}[t]
  \caption{BREC, all $400$ pairs, official RPC protocol. Fraction of pairs
           distinguished; all methods attain $100\%$ reliability. No method
           receives edge attributes here, so this is a comparison of
           structural descriptors under matched conditions.
           $\dagger$ marks methods whose only structural input is a per-edge
           function of $(\egirth{e}, \mult{e})$, and which
           Proposition~\ref{prop:degeneracy} therefore constrains.}
  \label{tab:brec}
  \centering
  \small
  \begin{tabular}{lccccccc}
    \toprule
    Method & Overall & Basic & Regular & Extension & CFI & 4-Vert. & Dist.-Reg. \\
           & (400)   & (60)  & (100)   & (100)     & (100) & (20)  & (20) \\
    \midrule
    Exact $2$-FWL                 & 0.675 & 1.000 & 0.500 & 1.000 & 0.600 & 0.000 & 0.000 \\
    PPGN                          & 0.518 & 1.000 & 0.500 & 0.970 & 0.000 & 0.000 & 0.000 \\
    \EGAGNN{}$^\dagger$           & 0.485 & 1.000 & 0.490 & 0.820 & 0.030 & 0.000 & 0.000 \\
    Edge-girth multiset$^\dagger$ & 0.483 & 1.000 & 0.490 & 0.810 & 0.030 & 0.000 & 0.000 \\
    GSN                           & 0.443 & 0.950 & 0.490 & 0.710 & 0.000 & 0.000 & 0.000 \\
    Triangle counts               & 0.398 & 0.983 & 0.480 & 0.520 & 0.000 & 0.000 & 0.000 \\
    GCN / GIN / GatedGCN          & 0.000 & 0.000 & 0.000 & 0.000 & 0.000 & 0.000 & 0.000 \\
    \bottomrule
  \end{tabular}
\end{table}

On BREC no method receives edge attributes, so Table~\ref{tab:brec} compares
structural descriptors on equal terms. Edge-girth separates every
\textsc{Basic} pair and $82\%$ of \textsc{Extension}, placing \EGAGNN{} ahead
of GSN and of triangle counting overall, though below exact $2$-FWL and PPGN.
The margin over GSN is clearest on \textsc{Extension}, $0.82$ against $0.71$.

One observation runs against the design of the model. \EGAGNN{} and a direct
hash of the per-edge multiset $\{\!\{(\egirth{e}, \mult{e})\}\!\}$ differ by a
single pair out of four hundred. Propagating the descriptor through a network
buys nothing on this task: the feature, not the architecture, does all the
work. This contrasts with \textsc{Zinc}, where the architecture accounts for a
substantial share of the improvement, and
Section~\ref{sec:expressivity} explains why no architecture built on this
descriptor could have done better here.

\paragraph{The predicted blind spot, pair by pair.}
Aggregate rates understate the point, since
Proposition~\ref{prop:degeneracy} constrains individual pairs. Of the $400$
BREC pairs, $90$ have both graphs edge-girth-regular with matching parameters
$(k, g, \lambda)$; these include the whole of the distance-regular and
four-vertex-condition categories, all $40$ of whose graphs we verify to be
edge-girth-regular. Table~\ref{tab:contingency} crosses this against the
outcome.

\begin{table}[t]
  \caption{Pairs resolved versus edge-girth-regularity, all $400$ BREC pairs.
           The upper-left cell is the one Proposition~\ref{prop:degeneracy}
           constrains; it is empty for both methods.}
  \label{tab:contingency}
  \centering
  \small
  \begin{tabular}{lcccc}
    \toprule
    & \multicolumn{2}{c}{\EGAGNN{}} & \multicolumn{2}{c}{Edge-girth multiset} \\
    \cmidrule(lr){2-3}\cmidrule(lr){4-5}
    & egr & not egr & egr & not egr \\
    \midrule
    resolved     & \textbf{0} & 194 & \textbf{0} & 193 \\
    not resolved & 90 & 116 & 90 & 117 \\
    \bottomrule
  \end{tabular}
\end{table}

Not one edge-girth-regular pair is resolved, by either method, anywhere in the
benchmark. The prediction holds without exception at the level of individual
pairs, not merely in aggregate.

The converse cell is populated as expected ($116$ pairs go unresolved although the
descriptor does vary on them), since the proposition is a
one-directional guarantee and not a characterisation of success. We return to
what it does and does not explain in Section~\ref{sec:discussion}.

\section{Discussion and Limitations}
\label{sec:discussion}
 
\paragraph{What Proposition~\ref{prop:degeneracy} does and does not explain.}
Two qualifications keep the BREC result from being over-read. Within the
\textsc{Regular} category, the $50$ edge-girth-regular pairs coincide
\emph{exactly} with the strongly regular ones, no pair separates the two
properties (Table~\ref{tab:regular-strat}). Edge-girth-based methods resolve
$49$ of the $50$ non-egr pairs and none of the $50$ egr ones, while exact
$2$-FWL resolves all $50$ non-egr pairs and likewise none of the egr ones: the
two coincide except for a single pair, which is neither edge-girth-regular nor
strongly regular and is resolved by $2$-FWL alone. That pair, \#71, is the
only instance in the entire category where our prediction and generic $3$-WL
hardness come apart. Second, the CFI pairs fall outside the proposition's
scope entirely: not one of the $100$ has both graphs edge-girth-regular with
matching parameters, so the descriptor does vary on them, and yet the
edge-girth methods resolve $3\%$ of them against $60\%$ for exact $2$-FWL
(Table~\ref{tab:brec}). Why an informative descriptor fails so completely
there is an empirical finding our analysis does not account for, and we
report it as such rather than stretch the theory to cover it. The proposition
is a one-directional guarantee, edge-girth-regularity implies failure and not a characterisation of success.
 
\paragraph{How much the unbounded descriptor is worth.}
The margin between edge-girth and the best bounded dictionary we tried is
small. A per-edge cycle count reaching length eight is within $8\%$ of
$(\egirth{e}, \mult{e})$ on \textsc{Zinc}, and the two standard deviations
nearly overlap. 
On this dataset the practical benefit of removing the ceiling
is modest once it is set correctly, the harder problem is that it cannot be
set correctly in advance (Section~\ref{subsec:descriptors}).
Whether a family of
graphs exists on which no affordable ceiling suffices is not settled by our
experiments, and constructing a benchmark that isolates it proved hard.
 
\paragraph{How far the supervised evidence reaches.}
The descriptor is computed independently of the objective, so nothing in its
construction is specific to regression; but our supervised evidence comes from
a single regression target on a single dataset, and we do not claim more.
Whether the gains reported here transfer to graph classification, or to
molecular targets whose dependence on cycle structure differs from
\textsc{Zinc}'s, is untested. We have made that single point of evidence as
solid as we could: baselines matched in parameter count and in nonlinear
depth, calibrated against published numbers, official splits, no test-set
model selection, and a descriptor study at full scale over four seeds. We also
ruled out the most natural shortcut explanation, that the cyclic term of
\textsc{Zinc}'s target is being read off the descriptor, by regressing the
target on the edge-girth histogram alone, which explains only
$R^2 \approx 0.26$ of its variance, far too little to account for the gap.
What remains is a mechanism we can bound but not identify: the gated
architecture alone reaches $0.204$ (Table~\ref{tab:descriptors}), the
descriptor contributes the rest, and why the interaction of per-edge cycle
structure with message passing is worth that much is open. 
 
Two comparisons in Table~\ref{tab:zinc} remain confounded and should not be
read as descriptor comparisons. GSN differs from \EGAGNN{} in backbone as well
as in descriptor, so its gap cannot be attributed to either;
Section~\ref{subsec:descriptors} is the controlled version of that comparison.
And GCN, GIN and GatedGCN do not consume \textsc{Zinc}'s bond types in our
implementation, so they stand as architecture baselines rather than as
chemistry-aware ones.
 
\paragraph{Directions.}
The most immediate follow-up is to widen the range of downstream tasks. Since
the descriptor is task-agnostic by construction, graph classification with a
matched architecture, and molecular targets whose dependence on cycle
structure differs from \textsc{Zinc}'s, would establish whether the gains
observed here are a property of the representation or of this particular
objective. Second, the regime the descriptor is designed for (graphs whose
informative cycles are long) remains untested; Tanner graphs of LDPC codes
are a natural candidate, since decoding performance degrades with short cycles
through an edge \cite{xu2025ldpc}, which is exactly what $\egirth{e}$
reports. Third, combining edge-based constraints with vertex-based ones such
as the degree sequence would tighten the prescribed local structure and may
leave the edge-girth-regular regime altogether. We note finally that building
a synthetic probe to isolate cycle scale is harder than it appears: in our
attempts, a target simple enough to control was either predictable from
adjacency alone, without any cycle-length information, or an invertible
function of the descriptor.

\printbibliography

@article{maron2019provably,
  title={Provably powerful graph networks},
  author={Maron, Haggai and Ben-Hamu, Heli and Serviansky, Hadar and Lipman, Yaron},
  journal={Advances in neural information processing systems},
  volume={32},
  year={2019}
}

@article{xu2018powerful,
  title={How powerful are graph neural networks?},
  author={Xu, Keyulu and Hu, Weihua and Leskovec, Jure and Jegelka, Stefanie},
  journal={arXiv preprint arXiv:1810.00826},
  year={2018}
}

@inproceedings{morris2019weisfeiler,
  title={Weisfeiler and leman go neural: Higher-order graph neural networks},
  author={Morris, Christopher and Ritzert, Martin and Fey, Matthias and Hamilton, William L and Lenssen, Jan Eric and Rattan, Gaurav and Grohe, Martin},
  booktitle={Proceedings of the AAAI conference on artificial intelligence},
  volume={33},
  pages={4602--4609},
  year={2019}
}

@article{bouritsas2022improving,
  title={Improving graph neural network expressivity via subgraph isomorphism counting},
  author={Bouritsas, Giorgos and Frasca, Fabrizio and Zafeiriou, Stefanos and Bronstein, Michael M},
  journal={IEEE Transactions on Pattern Analysis and Machine Intelligence},
  volume={45},
  number={1},
  pages={657--668},
  year={2022},
  publisher={IEEE}
}

@article{dwivedi2021graph,
  title={Graph neural networks with learnable structural and positional representations},
  author={Dwivedi, Vijay Prakash and Luu, Anh Tuan and Laurent, Thomas and Bengio, Yoshua and Bresson, Xavier},
  journal={arXiv preprint arXiv:2110.07875},
  year={2021}
}

@article{marey2026realizability,
  title={On the Realizability of Edge-Girth Sequences},
  author={Marey, Lilian and Hilaire, Paul and Laclau, Charlotte},
  journal={arXiv preprint arXiv:2607.25629},
  year={2026}
}

@article{jajcay2018edge,
  title={Edge-girth-regular graphs},
  author={Jajcay, R. and  Kiss, G. and Miklavic, S},
  journal={European Journal of Combinatorics},
  volume={72},
  pages={70--82},
  year={2018},
  publisher={Elsevier}
}

@article{goedgebeur2025exhaustive,
  title={Exhaustive generation of edge-girth-regular graphs},
  author={Goedgebeur, Jan and Jooken, Jorik},
  journal={Experimental Mathematics},
  pages={1--13},
  year={2025},
  publisher={Taylor \& Francis}
}

@incollection{brouwer2011distance,
  title={Distance-regular graphs},
  author={Brouwer, Andries E and Haemers, Willem H},
  booktitle={Spectra of graphs},
  pages={177--185},
  year={2011},
  publisher={Springer}
}

@article{van2014distance,
  title={Distance-regular graphs},
  author={Van Dam, Edwin R and Koolen, Jack H and Tanaka, Hajime},
  journal={arXiv preprint arXiv:1410.6294},
  year={2014}
}

@inproceedings{bernstein2010nearly,
  title={A nearly optimal algorithm for approximating replacement paths and k shortest simple paths in general graphs},
  author={Bernstein, Aaron},
  booktitle={Proceedings of the twenty-first annual ACM-SIAM symposium on Discrete Algorithms},
  pages={742--755},
  year={2010},
  organization={SIAM}
}

@article{woodhouse2016stochastic,
  title={Stochastic cycle selection in active flow networks},
  author={Woodhouse, Francis G and Forrow, Aden and Fawcett, Joanna B and Dunkel, J{\"o}rn},
  journal={Proceedings of the national academy of sciences},
  volume={113},
  number={29},
  pages={8200--8205},
  year={2016},
  publisher={National Academy of Sciences}
}

@article{xu2025ldpc,
  title={LDPC Codes on Balanced Incomplete Block Designs: Construction, Girth, and Cycle Structure Analysis},
  author={Xu, Hengzhou and Zhang, Xiaodong and Xu, Mengmeng and Yu, Haipeng and Zhu, Hai},
  journal={Entropy},
  volume={27},
  number={5},
  pages={476},
  year={2025},
  publisher={MDPI}
}

@article{wang2023empirical,
  title={An empirical study of realized gnn expressiveness},
  author={Wang, Yanbo and Zhang, Muhan},
  journal={arXiv preprint arXiv:2304.07702},
  year={2023}
}

@article{dwivedi2023benchmarking,
  title={Benchmarking graph neural networks},
  author={Dwivedi, Vijay Prakash and Joshi, Chaitanya K and Luu, Anh Tuan and Laurent, Thomas and Bengio, Yoshua and Bresson, Xavier},
  journal={Journal of Machine Learning Research},
  volume={24},
  number={43},
  pages={1--48},
  year={2023}
}

@article{flum2004parameterized,
  title={The parameterized complexity of counting problems},
  author={Flum, J{\"o}rg and Grohe, Martin},
  journal={SIAM Journal on Computing},
  volume={33},
  number={4},
  pages={892--922},
  year={2004},
  publisher={SIAM}
}

@article{bodnar2021weisfeiler,
  title={Weisfeiler and lehman go cellular: Cw networks},
  author={Bodnar, Cristian and Frasca, Fabrizio and Otter, Nina and Wang, Yuguang and Lio, Pietro and Montufar, Guido F and Bronstein, Michael},
  journal={Advances in neural information processing systems},
  volume={34},
  pages={2625--2640},
  year={2021}
}

@article{huang2022boosting,
  title={Boosting the cycle counting power of graph neural networks with I2-GNNs},
  author={Huang, Yinan and Peng, Xingang and Ma, Jianzhu and Zhang, Muhan},
  journal={arXiv preprint arXiv:2210.13978},
  year={2022}
}


\appendix
 
\section{Experimental Details}
\label{app:details}
 
\subsection{Hyperparameters}
 
\begin{table}[h]
  \caption{Training configuration. Values are shared across methods unless
           stated otherwise in Section~\ref{app:asymmetries}.}
  \label{tab:hyperparams}
  \centering
  \small
  \begin{tabular}{lccc}
    \toprule
    & \textsc{Zinc} & Descriptor study & BREC \\
    \midrule
    Optimizer            & Adam & Adam & Adam \\
    Learning rate        & $10^{-3}$ & $10^{-3}$ & $10^{-4}$ \\
    Weight decay         & $0$ & $0$ & $10^{-4}$ \\
    Batch size (train)   & 32 & 32 & 16 \\
    Batch size (eval)    & 64 & 64 & --- \\
    Scheduler            & none & none & \texttt{ReduceLROnPlateau} \\
    Propagation layers   & 4 & 4 & 3 \\
    Hidden dim.          & matched & matched & 32 \\
    Output dim.          & = hidden & = hidden & 16 \\
    Aggregation          & sum & sum & sum \\
    Readout              & sum & sum & sum \\
    Loss                 & $L_1$ & $L_1$ & cosine embedding \\
    Epochs               & 200 & 200 & 20 per pair \\
    Seeds                & 13--16 & 13--16 & 13 \\
    \bottomrule
  \end{tabular}
\end{table}
 
The BREC scheduler is \texttt{ReduceLROnPlateau} in its default PyTorch
configuration ($\mathrm{mode}=\mathrm{min}$, factor $0.1$, patience $10$), and
the cosine embedding loss uses margin $0$. BREC is run at a single seed; the
$32$ relabelings of the RPC protocol supply the stochasticity the test
requires, so per-seed variance is not the quantity of interest there. On BREC
and CSL, where the graphs carry no attributes, node features are initialised
to a constant vector of ones.
 
\subsection{Descriptor construction}
 
Bridges are represented internally as $(\egirth{e}, \mult{e}) = (\infty, 0)$
and injected as the triple $(0, 0, 1)$: the two structural channels are zeroed
and the indicator is set, so the network never receives a numerical stand-in
for infinity. Finite values are standardised (mean and standard deviation, the
latter floored to $1$ when degenerate) over the first $300$ graphs of the
training split, pooled over all their edges. On CSL, which has only $150$
graphs in total, the normaliser is fitted on the whole dataset. BREC has no
train/test split, so the $300$ graphs used there are the first entries of
whichever subset a given invocation preprocessed; since the normalisation is
applied identically to both graphs of every pair and the RPC test is
scale-free within a pair, this does not affect the reported outcomes.
 
\texttt{edge\_index} stores both directions of every undirected edge. The
descriptor is looked up by an unordered key, so the two directed copies carry
identical values by construction; each undirected edge consequently
contributes twice to the neighbourhood aggregation of Equation~\eqref{eq:nodeupdate},
which is a constant factor under sum aggregation.
 
\subsection{Parameter-budget matching}
\label{app:budget}
 
All \textsc{Zinc} architectures are matched to $100{,}000$ parameters within
$\pm 10\%$ by binary search on the hidden width over $[4, 1024]$, the
parameter count being monotone in that width for every architecture used. The
search is run independently for each method and each descriptor variant, so a
variant supplying more structural channels, which enlarges the edge
embedding and the edge-update MLP, is compensated by a smaller hidden
width.
 
\begin{table}[h]
  \caption{Effective hidden width and parameter count. Target $100{,}000$,
           tolerance $\pm 10\%$.}
  \label{tab:budget}
  \centering
  \small
  \begin{tabular}{lcc|lcc}
    \toprule
    Method & Hidden & Params & Descriptor variant & Hidden & Params \\
    \midrule
    \EGAGNN{}    & 54  & 104{,}113 & none (constant)          & 54 & 104{,}113 \\
    GCN          & 130 & 106{,}081 & triangle count           & 54 & 103{,}573 \\
    GIN          & 98  & \phantom{0}99{,}961 & cycles $3$--$4$ & 54 & 103{,}843 \\
    GatedGCN     & 74  & 102{,}121 & cycles $3$--$6$          & 54 & 104{,}383 \\
    GatedGCN-MLP & 54  & 102{,}601 & cycles $3$--$8$          & 50 & \phantom{0}90{,}351 \\
    GSN ($k=6$)  & 98  & 102{,}313 & $(\egirth{e},\mult{e})$  & 54 & 104{,}113 \\
    \bottomrule
  \end{tabular}
\end{table}
 
One entry deserves comment. The cycles-$3$--$8$ variant lands at $90{,}351$
parameters, $9.6\%$ below target and $13\%$ below the edge-girth variant it is
compared against: the binary search overshot downwards when compensating for
its six structural channels. That variant is therefore the one comparison in
Table~\ref{tab:descriptors} run at a mild disadvantage, and its reported MAE
of $0.1005$ should be read as an upper bound on what a length-eight dictionary
achieves at full budget. This does not affect the conclusions drawn in
Section~\ref{subsec:descriptors}, which concern the dictionaries capped at
four and six.
 
\subsection{BREC protocol}
 
We use $32$ relabelings per graph, $20$ training epochs per pair, an embedding
dimension of $16$, a Hotelling $T^2$ threshold of $72.34$ and a loss-based
early stop at $0.2$, following the reference implementation. The reliability
control group uses the same number of relabelings and the same threshold as
the main test; a pair is credited only when the main statistic exceeds the
threshold and the control statistic does not.
 
The benchmark's $400$ pairs are partitioned into \textsc{Basic} ($60$),
\textsc{Regular} ($100$), \textsc{Extension} ($100$), \textsc{CFI} ($100$),
\textsc{4-Vertex-Condition} ($20$) and \textsc{Distance-Regular} ($20$). The
official release reports a single \textsc{Regular} bucket of $140$ pairs; our
last two categories (\textsc{4-Vertex-Condition} and \textsc{Distance-Regular})
partition exactly that bucket, so the totals agree and only the reporting
granularity differs.
 
Two of the evaluated methods are not learned. Exact $2$-FWL compares the two
graphs' stable pair-colourings directly. The edge-girth multiset baseline maps
each distinct $(\egirth{e}, \mult{e})$ value to a fixed pseudo-random vector
and sums over edges, giving one deterministic embedding per graph and no
parameters. Both skip the training loop entirely and are then passed through
the same $T^2$ test at the same threshold as the learned methods, so no
separate similarity proxy is used anywhere in Table~\ref{tab:brec}.
 
\subsection{Asymmetries between methods}
\label{app:asymmetries}
 
Three departures from the shared configuration should be recorded. On
\textsc{Zinc}, GCN, GIN, GatedGCN and GatedGCN-MLP do not consume bond types
in our implementation, whereas \EGAGNN{} and GSN do; this is discussed in
Section~\ref{sec:discussion}. On BREC, PPGN is built with half the hidden
width ($16$) and two layers instead of three, following its reference
configuration for memory reasons. Also on BREC, the triangle-count baseline
supplies a single structural channel rather than the shared edge-input
dimension.
 
\subsection{Hardware and software}
 
All experiments run on CPU: an $11$-core Apple M3 Pro with $18$\,GB of RAM,
with ten worker processes. Total wall-clock time was $6.5$\,h for the full
BREC campaign across nine methods, $67$\,min for the ablation and $57$\,min
for the descriptor study, each over four seeds.
 
\section{Full Result Tables}
\label{app:tables}
 
\begin{table}[h]
  \caption{BREC, all $400$ pairs, all evaluated methods. Every method attains
           a $100\%$ reliability rate ($0$ control-group failures out of
           $400$).}
  \label{tab:brec-full}
  \centering
  \small
  \begin{tabular}{lccccccc}
    \toprule
    Method & Basic & Regular & Extension & CFI & 4-Vert. & Dist.-Reg. & Overall \\
           & (60)  & (100)   & (100)     & (100) & (20)  & (20)       & (400) \\
    \midrule
    Exact $2$-FWL       & 1.000 & 0.500 & 1.000 & 0.600 & 0.000 & 0.000 & 0.675 \\
    PPGN                & 1.000 & 0.500 & 0.970 & 0.000 & 0.000 & 0.000 & 0.518 \\
    \EGAGNN{}           & 1.000 & 0.490 & 0.820 & 0.030 & 0.000 & 0.000 & 0.485 \\
    Edge-girth multiset & 1.000 & 0.490 & 0.810 & 0.030 & 0.000 & 0.000 & 0.483 \\
    GSN                 & 0.950 & 0.490 & 0.710 & 0.000 & 0.000 & 0.000 & 0.443 \\
    Triangle counts     & 0.983 & 0.480 & 0.520 & 0.000 & 0.000 & 0.000 & 0.398 \\
    GCN                 & 0.000 & 0.000 & 0.000 & 0.000 & 0.000 & 0.000 & 0.000 \\
    GIN                 & 0.000 & 0.000 & 0.000 & 0.000 & 0.000 & 0.000 & 0.000 \\
    GatedGCN            & 0.000 & 0.000 & 0.000 & 0.000 & 0.000 & 0.000 & 0.000 \\
    \bottomrule
  \end{tabular}
\end{table}
 
\begin{table}[h]
  \caption{\textsc{Zinc} test MAE per seed. GSN rows without bond types are
           reported for reference; the main text uses the bond-aware version.}
  \label{tab:zinc-seeds}
  \centering
  \small
  \begin{tabular}{lcccc}
    \toprule
    Method & Seed 13 & Seed 14 & Seed 15 & Seed 16 \\
    \midrule
    \EGAGNN{}                & 0.0958 & 0.0879 & 0.0923 & 0.0969 \\
    GSN ($k=6$, bond types)  & 0.1702 & 0.1926 & 0.1706 & 0.1566 \\
    GSN ($k=4$, bond types)  & 0.2651 & 0.2366 & 0.2477 & 0.2393 \\
    GSN ($k=6$, no bonds)    & 0.2222 & 0.2073 & 0.2257 & 0.2119 \\
    GSN ($k=4$, no bonds)    & 0.3390 & 0.3162 & 0.3019 & 0.3224 \\
    GatedGCN-MLP             & 0.2726 & 0.2748 & 0.2926 & 0.2501 \\
    GatedGCN                 & 0.2730 & 0.2859 & 0.2838 & 0.2687 \\
    GIN                      & 0.3078 & 0.3200 & 0.3027 & 0.3305 \\
    GCN                      & 0.4329 & 0.4537 & 0.4303 & 0.4319 \\
    \bottomrule
  \end{tabular}
\end{table}
 
\begin{table}[h]
  \caption{Descriptor study and ablation, \textsc{Zinc} test MAE per seed.
           Architecture and bond attributes identical throughout.}
  \label{tab:descriptors-seeds}
  \centering
  \small
  \begin{tabular}{lcccc}
    \toprule
    Structural descriptor & Seed 13 & Seed 14 & Seed 15 & Seed 16 \\
    \midrule
    none (constant)                    & 0.2002 & 0.2066 & 0.2031 & 0.2077 \\
    Gaussian noise                     & 0.2182 & 0.2393 & 0.2541 & 0.2447 \\
    triangle count                     & 0.2164 & 0.2320 & 0.2218 & 0.1913 \\
    cycles $3$--$4$                    & 0.2085 & 0.1974 & 0.2167 & 0.1857 \\
    cycles $3$--$6$                    & 0.1307 & 0.1209 & 0.0971 & 0.1113 \\
    cycles $3$--$8$                    & 0.1146 & 0.0904 & 0.0959 & 0.1010 \\
    $\mult{e}$ only                    & 0.1813 & 0.1328 & 0.1289 & 0.1769 \\
    $\egirth{e}$ only                  & 0.0890 & 0.0973 & 0.0882 & 0.0908 \\
    $(\egirth{e}, \mult{e})$ + bridge  & 0.0958 & 0.0879 & 0.0923 & 0.0969 \\
    \bottomrule
  \end{tabular}
\end{table}
 
\begin{table}[h]
  \caption{Contingency under the criterion actually used (both
           $\egirth{e}$ and $\mult{e}$ constant) and under the weaker
           criterion that constrains $\egirth{e}$ alone. The weaker criterion
           admits $103$ pairs instead of $90$, and $12$ of the additional $13$
           are resolved, which is why Remark~\ref{rem:lambda} states that a
           girth-only version of Proposition~\ref{prop:degeneracy} would be
           false.}
  \label{tab:contingency-full}
  \centering
  \small
  \begin{tabular}{llcccc}
    \toprule
    Criterion & Method & res.\ $\wedge$ egr & res.\ $\wedge$ $\neg$egr
              & $\neg$res.\ $\wedge$ egr & $\neg$res.\ $\wedge$ $\neg$egr \\
    \midrule
    $(\egirth{e}, \mult{e})$ & \EGAGNN{}           & \textbf{0} & 194 & 90 & 116 \\
    $(\egirth{e}, \mult{e})$ & Edge-girth multiset & \textbf{0} & 193 & 90 & 117 \\
    \midrule
    $\egirth{e}$ only        & \EGAGNN{}           & 12 & 182 & 91 & 115 \\
    $\egirth{e}$ only        & Edge-girth multiset & 12 & 181 & 91 & 116 \\
    \bottomrule
  \end{tabular}
\end{table}
 
\begin{table}[h]
  \caption{The \textsc{Regular} category, stratified by strong regularity.
           Edge-girth-regularity and strong regularity coincide on all $100$
           pairs: every edge-girth-regular pair is strongly regular and
           conversely, so no pair of this category separates the two
           properties.}
  \label{tab:regular-strat}
  \centering
  \small
  \begin{tabular}{lcccc}
    \toprule
    Subset ($n=50$ each) & res.\ $\wedge$ egr & res.\ $\wedge$ $\neg$egr
                         & $\neg$res.\ $\wedge$ egr & $\neg$res.\ $\wedge$ $\neg$egr \\
    \midrule
    Plain regular (not strongly regular) & 0 & 49 & \phantom{0}0 & 1 \\
    Strongly regular                     & 0 & \phantom{0}0 & 50 & 0 \\
    \bottomrule
  \end{tabular}
\end{table}
 
\section{CSL}
\label{app:csl}
 
We report CSL for completeness, but exclude it from the main text because two
reference implementations proved unreliable on it.
 
\begin{table}[h]
  \caption{CSL, $150$ graphs, $10$ classes, five-fold cross-validation, $100$
           epochs. Chance level is $0.100$. Every method returns the same
           accuracy on all five folds, so all standard deviations are zero.
           Starred rows are unreliable and should not be read as results.}
  \label{tab:csl}
  \centering
  \small
  \begin{tabular}{lc}
    \toprule
    Method & Test accuracy \\
    \midrule
    \EGAGNN{}                & 0.300 \\
    GSN                      & 0.300 \\
    Edge-girth multiset      & 0.300 \\
    Exact $2$-FWL$^{\star}$  & 0.200 \\
    GCN                      & 0.100 \\
    GIN                      & 0.100 \\
    GatedGCN                 & 0.100 \\
    PPGN$^{\star}$           & 0.100 \\
    \bottomrule
  \end{tabular}
\end{table}
 
Training uses Adam with learning rate $10^{-3}$ and weight decay $10^{-5}$,
cross-entropy loss, batch size $16$, three propagation layers of width $32$,
and constant node features. Folds are stratified with seed $13$.
 
\paragraph{Why PPGN is unreliable here.}
Its training loss moves (from $1962$ at the first epoch to $2.73$ at the
hundredth) but never falls below the chance-level cross-entropy
$\ln 10 \approx 2.303$; the minimum reached over $100$ epochs is $2.462$, at
epoch $87$. The model trains without learning a discriminative representation,
and its accuracy of exactly $0.100$ carries no information either way.
 
\paragraph{Why exact $2$-FWL is unreliable here.}
On BREC, our $2$-FWL implementation compares two graphs' jointly canonicalised
stable pair-colourings for exact multiset equality: a direct, exact, unlearned
pairwise decision. CSL is a classification task and requires a per-graph
vector, so the colouring is reduced to the sorted histogram of stable-colour
class sizes and passed to a logistic regression. That histogram is a lossy
summary (two $2$-FWL-inequivalent graphs can in principle share one) and
the prediction depends on a classifier trained on other graphs. The resulting
$0.200$ therefore measures the summary and the classifier, not the separating
power of $3$-WL, and we do not report it as a $3$-WL baseline.
 
The uniform absence of fold-to-fold variance in Table~\ref{tab:csl} is
consistent with the tasks these methods solve on CSL being deterministic: each
method either separates a given skip-length class or does not, and the
stratified folds contain the same classes.
\end{document}